\documentclass[11pt,a4paper]{article}
\usepackage{hyperref}
\usepackage{naaclhlt2018}
\usepackage{times}
\usepackage{latexsym}
\usepackage{amsmath,amsthm,amssymb}
\usepackage{bm}
\usepackage{bbm}
\usepackage{graphicx}
\usepackage{threeparttable}
\usepackage{multirow}
\usepackage{mathrsfs}
\usepackage{url}
\usepackage{algorithm}
\usepackage[noend]{algpseudocode}
\usepackage[title]{appendix}
\usepackage{lipsum}

\aclfinalcopy 
\def\aclpaperid{156} 

\algrenewcommand\algorithmicindent{1.0em}
\newcommand{\secref}[1]{\S\ref{#1}}

\newtheorem{theorem}{Theorem}
\newtheorem{corollary}{Corollary}
\newtheorem{lemma}{Lemma}
 
\makeatletter
\def\blfootnote{\gdef\@thefnmark{}\@footnotetext}
\makeatother

\newcommand{\cfs}{{\cal F}_s}
\newcommand{\cfd}{{\cal F}_d}
\newcommand{\cfdi}{{\cal F}_{d_i}}
\newcommand{\fvec}{{\bf f}}
\newcommand{\cc}{{\cal C}}
\newcommand{\cd}{{\cal D}}
\newcommand{\cl}{{\cal L}}
\newcommand{\expe}{\mathop{{}\mathbb{E}}}
\algnewcommand{\LeftComment}[1]{\State \(\triangleright\) #1}
\newcommand{\pluseq}{\mathrel{+}=}

\newcommand{\mli}[1]{\mathit{#1}}
\DeclareMathOperator*{\argmin}{arg\,min}
\DeclareMathOperator*{\argmax}{arg\,max}

\newcommand{\amdc}{\texttt{MAN}}

\title{Multinomial Adversarial Networks for Multi-Domain Text Classification}

\author{Xilun Chen\\
  Department of Computer Science\\
  Cornell Unversity\\
  Ithaca, NY, 14853, USA\\
  {\tt xlchen@cs.cornell.edu}\\\And
  Claire Cardie\\
  Department of Computer Science\\
  Cornell Unversity\\
  Ithaca, NY, 14853, USA\\
  {\tt cardie@cs.cornell.edu}\\
}

\date{}

\begin{document}
\maketitle
\begin{abstract}
Many text classification tasks are known to be highly domain-dependent. 
Unfortunately, the availability of training data can vary drastically across domains.
Worse still, for some domains there may not be any annotated data at all.
In this work, we propose a \emph{multinomial adversarial network} (\amdc{}) to tackle the text classification problem in this real-world multi-domain setting (MDTC).
We provide theoretical justifications for the \amdc{} framework, proving that different instances of \amdc{}s are essentially minimizers of various f-divergence metrics~\cite{10.2307/2984279} among \emph{multiple} probability distributions.
\amdc{}s are thus a theoretically sound generalization of traditional adversarial networks that discriminate over \emph{two} distributions.
More specifically, for the MDTC task, \amdc{} learns features that are invariant across multiple domains 
by resorting to its ability to reduce the divergence among the feature distributions of each domain.
We present experimental results showing that \amdc{}s significantly outperform the prior art on the MDTC task.
We also show that \amdc{}s achieve state-of-the-art performance for domains with no labeled data.
\end{abstract}

\blfootnote{The source code of \amdc{} can be found at \url{https://github.com/ccsasuke/man}}
\section{Introduction}
Text classification is one of the most fundamental tasks in Natural Language Processing, and has found its way into a wide spectrum of NLP applications, ranging from email spam detection and social media analytics to sentiment analysis and data mining.
Over the past couple of decades, supervised statistical learning methods have become the dominant approach for text classification (e.g.~\newcite{mccallum1998comparison,D14-1181,P15-1162}).
Unfortunately, many text classification tasks are highly domain-dependent in that a text classifier trained using labeled data from one domain is likely to perform poorly on another.
In the task of sentiment classification, for example, 
a phrase ``runs fast'' is usually associated with positive sentiment in the sports domain; not so when a user is reviewing the battery of an electronic device.
In real applications, therefore, an adequate amount of training data from 
each domain of interest is typically required, and this is expensive to obtain.

Two major lines of work attempt to tackle this challenge: \textbf{domain adaptation}~\cite{P07-1056} and \textbf{multi-domain text classification} (MDTC)~\cite{P08-2065}.
In domain adaptation, the assumption is that there is some domain with abundant training data (the source domain), and the goal is to utilize knowledge learned from the source domain to help perform classifications on another lower-resourced target domain.\footnote{Review \secref{sec:relatedwork} for other variants of domain adaptation.}
The focus of this work, MDTC, instead simulates an arguably more realistic scenario, where labeled data may exist for multiple domains, but in insufficient amounts to train an effective classifier for one or more of the domains.
Worse still, some domains may have \emph{no} labeled data at all.
The objective of MDTC is to leverage all the available resources in order to improve the system performance over all domains simultaneously.

One state-of-the-art system for MDTC, the CMSC of~\newcite{wu2015}, combines a
 classifier that is shared across all domains (for learning domain-invariant knowledge) 
 with a set of classifiers, one per domain, each of which captures
 domain-specific text classification knowledge.
This paradigm is sometimes known as the Shared-Private model~\cite{NIPS2016_6254}.
CMSC, however, lacks an explicit mechanism to ensure that the shared classifier 
captures only domain-independent knowledge:  
the shared classifier may well also acquire some domain-specific features 
that are useful for a subset of the domains.
We hypothesize that better performance can be obtained if this constraint were 
explicitly enforced.  

In this paper, we thus propose Multinomial 
Adversarial Networks (henceforth, \amdc{}s) for the task of multi-domain
text classification. 
In contrast to standard adversarial networks \cite{NIPS2014_5423}, which
serve as a tool for minimizing
the divergence between \emph{two} distributions~\cite{nowozin2016f},  
\amdc{}s represent a family of theoretically sound adversarial networks that, in contrast,
leverage a \emph{multinomial discriminator} to directly minimize the 
divergence among multiple probability distributions.
And just as binomial adversarial networks have been applied 
to numerous tasks 
(e.g.\ image generation~\cite{NIPS2014_5423},
domain adaptation~\cite{ganin2016domain}, cross-lingual sentiment analysis~\cite{2016arXiv160601614C}),
we anticipate that \amdc{}s will make a versatile machine learning 
framework with applications beyond the MDTC task studied in this work.

We introduce the \amdc{} architecture in \secref{sec:model} and prove in \secref{sec:model_proof} that it 
 directly minimizes the (generalized) f-divergence among multiple distributions 
so that they are indistinguishable upon successful training.
Specifically for MDTC, \amdc{} is used to overcome the aforementioned limitation in prior art where domain-specific features may sneak into the shared model.
This is done by relying on \amdc{}'s power of minimizing the divergence among the feature distributions of each domain.
The high-level idea is that \amdc{} will make the extracted feature distributions of each domain indistinguishable from one another, thus learning general features that are invariant across domains.

We then validate the effectiveness of \amdc{} in experiments on two MDTC data sets.
We find first that \amdc{} significantly outperforms the state-of-the-art 
CMSC method~\cite{wu2015} on the widely used multi-domain Amazon review dataset,
and does so without relying on external resources such as sentiment 
lexica (\secref{sec:mdtc_exp}). 
When applied to the FDU-MTL dataset (\secref{sec:mtl_exp}), we
obtain similar results: \amdc{} achieves substantially higher accuracy than the 
previous top-performing method, ASP-MTL~\cite{P17-1001}.
ASP-MTL is the first empirical attempt to use a multinomial adversarial network proposed for a multi-task learning setting,
but is more restricted and can be viewed as a special case of \amdc{}.
In addition, we for the first time provide theoretical guarantees for \amdc{} (\secref{sec:model_proof}) that were absent in ASP-MTL.
Finally, while many MDTC methods such as CMSC require labeled data for each domain, \amdc{}s
can be applied in cases where no labeled data exists for
a subset of domains.  To evaluate \amdc{} in this semi-supervised setting, we 
compare \amdc{} to a method that can accommodate unlabeled data for (only) one domain~\cite{DBLP:journals/corr/ZhaoZWCMG17}, and show that \amdc{} achieves 
 performance comparable to the state of the art (\secref{sec:multisource_exp}).

\section{Model}\label{sec:model}

In this paper, we strive to tackle the text classification problem in a real-world setting
in which texts come from a variety of domains, each with a varying amount of labeled data.
Specifically, assume we have a total $N$ domains, $N_1$ \emph{labeled
domains} (denoted as $\Delta_L$) for which there is some labeled data, and $N_2$ \emph{unlabeled domains} ($\Delta_U$) for which no annotated training instances are
available.
Denote $\Delta = \Delta_L \cup \Delta_U$ as the collection of all domains, with $N = N_1+N_2$ being the total number of domains we are faced with.
The goal of this work, and MDTC in general, is to improve the overall classification performance across all $N$ domains, measured in this paper as the average classification accuracy across the $N$ domains in $\Delta$.

\subsection{Model Architecture}\label{sec:model_arch}
\begin{figure}[t]
    \centering
    \includegraphics[width=0.48\textwidth]{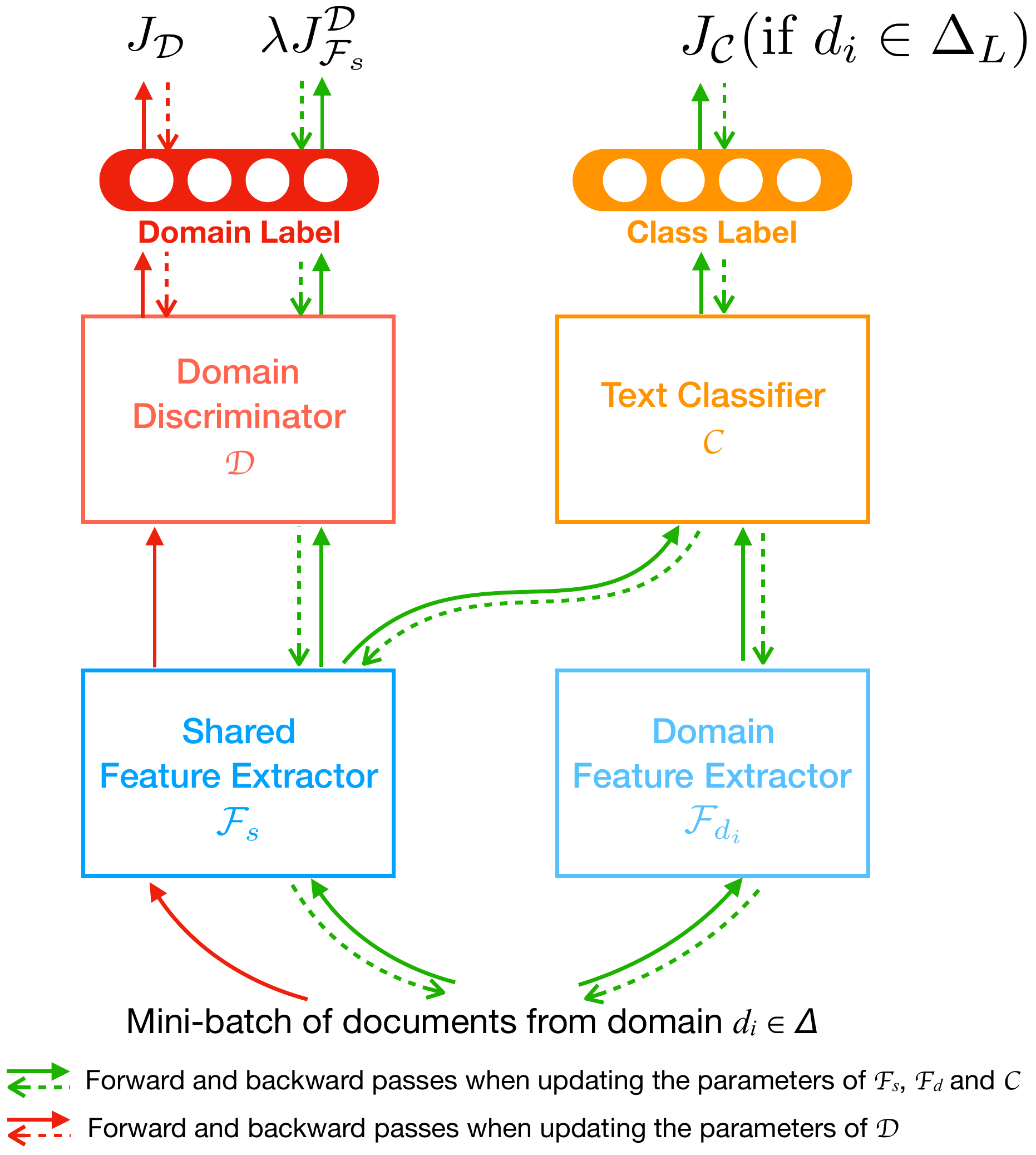}
    \vspace*{-7mm}
    \caption{
    \amdc{} for MDTC.
    The figure demonstrates the training on a mini-batch of data from one domain.
    One training iteration consists of one such mini-batch training from each domain.
    The parameters of $\cfs, \cfd, \cc$ are updated together, and the training flows are illustrated by the green arrows.
    The parameters of $\cd$ are updated separately, shown in red arrows.
    Solid lines indicate forward passes while dotted lines are backward passes.
    $J_{\cfs}^\cd$ is the domain loss for $\cfs$, which is anticorrelated with $J_\cd$ (e.g. $J_{\cfs}^\cd=-J_\cd$). (See \secref{sec:model},\secref{sec:model_proof})
    }
    \label{fig:amdc}
\end{figure}

As shown in Figure~\ref{fig:amdc}, the Multinomial Adversarial Network (\amdc{}) 
adopts the Shared-Private paradigm of \newcite{NIPS2016_6254} and
consists of four components: a \emph{shared feature extractor} $\cfs$, a \emph{domain feature extractor} $\cfdi$ for each labeled domain $d_i\in \Delta_L$, a \emph{text classifier} $\cc$, and finally a \emph{domain discriminator} $\cd$.
The main idea of \amdc{} is to explicitly model the domain-invariant features that are beneficial to the main classification task across all domains (i.e. \emph{shared features}, extracted by $\cfs$), as well as the domain-specific features that mainly contribute to the classification in its own domain (\emph{domain features}, extracted by $\cfd$).
Here, the adversarial domain discriminator $\cd$ has a multinomial output that takes a shared feature vector and predicts the likelihood of that sample coming from each domain.
As seen in Figure~\ref{fig:amdc} during the training flow of $\cfs$ (green arrows), $\cfs$ aims to confuse $\cd$ by minimizing $J_{\cfs}^\cd$ that is anticorrelated to $J_\cd$ (detailed in \secref{sec:model_training}), so that $\cd$ cannot predict the domain of a sample given its shared features.
The intuition is that if even a strong discriminator $\cd$ cannot tell the domain of a sample from the extracted features, those features $\cfs$ learned are essentially domain invariant.
By enforcing domain-invariant features to be learned by $\cfs$, when trained jointly via backpropagation, the set of domain features extractors $\cfd$ will each learn domain-specific features beneficial within its own domain.

The architecture of each component is relatively flexible, and can be decided by the practitioners to suit their particular classification tasks.
For instance, the feature extractors can adopt the form of Convolutional Neural Nets (CNN), Recurrent Neural Nets (RNN), or a Multi-Layer Perceptron (MLP), depending on the input data (See~\secref{sec:experiments}).
The input of \amdc{} will also be dependent on the feature extractor choice.
The output of a (shared/domain) feature extractor is a fixed-length vector, which is considered the (shared/domain) hidden features of some given input text.
On the other hand,
the outputs of $\cc$ and $\cd$ are label probabilities for class and domain prediction, respectively.
For example, both $\cc$ and $\cd$ can be MLPs with a softmax layer on top.
In~\secref{sec:model_proof}, we provide alternative architectures for $\cd$ and their mathematical implications.
We now present detailed descriptions of the \amdc{} training in \secref{sec:model_training} as well as the theoretical grounds in \secref{sec:model_proof}.

\begin{algorithm}[t]
\small
\begin{algorithmic}[1]
\Require
labeled corpus $\mathbb{X}$; unlabeled corpus $\mathbb{U}$; Hyperpamameter $\lambda > 0$, $k \in \mathbb{N}$
\Repeat
\LeftComment{$\cd$ iterations}
\For{$diter = 1$ to $k$}
\State $l_\cd = 0$
\ForAll{$d \in \Delta$}\Comment{For all $N$ domains}
\State Sample a mini-batch $\bm{x} \sim \mathbb{U}_{d}$
\State $\bm{f}_s = \cfs(\bm{x})$ \Comment{Shared feature vector}
\State $l_\cd \pluseq J_\cd(\cd(\bm{f}_s); d)$ \Comment{Accumulate $\cd$ loss}
\EndFor
\State Update $\cd$ parameters using $\nabla l_\cd$
\EndFor

\LeftComment{Main iteration}
\State $loss = 0$
\ForAll{$d \in \Delta_L$}\Comment{For all labeled domains}
\State Sample a mini-batch $(\bm{x},\bm{y})  \sim \mathbb{X}_{d}$
\State $\bm{f}_s = \cfs(\bm{x})$ 
\State $\bm{f}_d = \cfd(\bm{x})$ \Comment{Domain feature vector}
\State $loss \pluseq J_\cc(\cc(\bm{f}_s, \bm{f}_d); \bm{y})$ \Comment{Compute $\cc$ loss}
\EndFor

\ForAll{$d \in \Delta$}\Comment{For all $N$ domains}
\State Sample a mini-batch $\bm{x} \sim \mathbb{U}_{d}$
\State $\bm{f}_s = \cfs(\bm{x})$ 
\State $loss \pluseq \lambda\cdot J_{\cfs}^\cd(\cd(\bm{f}_s); d)$ \Comment{Domain loss of $\cfs$}
\EndFor
\State Update $\cfs$, $\cfd$, $\cc$ parameters using $\nabla loss$
\Until{convergence}

\end{algorithmic}
\caption{\amdc{} Training}
\label{alg:training}
\end{algorithm}

\subsection{Training}\label{sec:model_training}
Denote the annotated corpus in a labeled domain $d_i\in\Delta_L$ as $\mathbb{X}_{i}$;
and $(x, y)\sim \mathbb{X}_{i}$ is a sample drawn from the labeled data in domain $d_i$,
where $x$ is the input and $y$ is the task label.
On the other hand, for any domain $d_{i'}\in \Delta$,
denote the unlabeled corpus as $\mathbb{U}_{i'}$.
Note for a labeled domain, one can use a separate unlabeled corpus or simply use the labeled data (or use both).

In Figure~\ref{fig:amdc}, the arrows illustrate the training flows of various components.
Due to the adversarial nature of the domain discriminator $\cd$, it is trained with a separate optimizer (red arrows), while the rest of the networks are updated with the main optimizer (green arrows).
$\cc$ is only trained on labeled domains, and it takes  as input the concatenation of the shared and domain feature vectors.
At test time for unlabeled domains with no $\cfd$, the domain features are set to the $\bf 0$ vector for $\cc$'s input.
On the contrary, $\cd$ only takes the shared features as input, for both labeled and unlabeled domains.
The \amdc{} training is described in Algorithm~\ref{alg:training}.

In Algorithm~\ref{alg:training}, $\cl_\cc$ and $\cl_\cd$ are the loss functions of the text classifier $\cc$ and the domain discriminator $\cd$, respectively.
As mentioned in ~\secref{sec:model_arch}, $\cc$ has a $softmax$ layer on top for classification.
We hence adopt the canonical negative log-likelihood (NLL) loss:
\begin{align}
    \cl_\cc(\hat{y}, y) = -\log P(\hat{y}=y)\label{eqn:lc}
\end{align}
where $y$ is the true label and $\hat{y}$ is the $softmax$ predictions.
For $\cd$, we consider two variants of \amdc{}.
The first one is to use the NLL loss same as $\cc$ which suits the classification task;
while another option is to use the Least-Square (L2) loss that was shown to be able to alleviate the gradient vanishing problem when using the NLL loss in the adversarial setting~\cite{Mao_2017_ICCV}:
\begin{align}
    \cl^{NLL}_\cd(\hat{d}, d) &= -\log P(\hat{d}=d)\label{eqn:ld-nll}\\
    \cl^{L2}_\cd(\hat{d}, d) &= \sum_{i=1}^N (\hat{d}_i-\mathbbm{1}_{\{d=i\}})^2 \label{eqn:ld-l2}
\end{align}
where $d$ is the domain index of some sample and $\hat{d}$ is the prediction.
Without loss of generality, we normalize $\hat{d}$ so that $\sum_{i=1}^N\hat{d}_i = 1$ and $\forall i: \hat{d}_i\geq 0$.

Therefore, the objectives of $\cc$ and $\cd$ that we are minimizing are:
\begin{align}
    J_\cc &= \sum_{i=1}^N \expe_{(x,y)\sim\mathbb{X}_i} \left[ \cl_\cc(\cc(\cfs(x), \cfd(x)); y)\right]\label{eqn:jc}\\
    J_\cd &= \sum_{i=1}^N \expe_{x\sim\mathbb{U}_i} \left[ \cl_\cd(\cd(\cfs(x)); d)\right]\label{eqn:jd}
\end{align}

For the feature extractors, the training of domain feature extractors is straightforward, as their sole objective is to help $\cc$ perform better within their own domain.
Hence, $J_{\cfd} = J_\cc$ for any domain $d$.
Finally, the shared feature extractor $\cfs$ has two objectives:
to help $\cc$ achieve higher accuracy,
and to make the feature distribution invariant across all domains.
It thus leads to the following bipartite loss:
$$
    J_{\cfs} = J_{\cfs}^\cc + \lambda\cdot J_{\cfs}^\cd \nonumber
$$
where $\lambda$ is a hyperparameter balancing the two parts.
$J_{\cfs}^\cd$ is the domain loss of $\cfs$ anticorrelated to $J_\cd$:
\begin{align}
(NLL) J_{\cfs}^\cd&=-J_\cd\label{eqn:jfs-domain-nll}\\
(L2) J_{\cfs}^\cd &= \sum_{i=1}^N \expe_{x\sim\mathbb{U}_i} \left[\sum_{j=1}^N (\cd_j(\cfs(x))-\frac{1}{N})^2 \right]\label{eqn:jfs-domain-l2}
\end{align}
If $\cd$ adopts the NLL loss (\ref{eqn:jfs-domain-nll}), the domain loss is simply $-J_\cd$.
For the L2 loss (\ref{eqn:jfs-domain-l2}), $J_{\cfs}^\cd$ intuitively translates to pushing $\cd$ to make random predictions.
See~\secref{sec:model_proof} for theoretical justifications.

\section{Theories of Multinomial Adversarial Networks}\label{sec:model_proof}
The binomial adversarial nets are known to have theoretical connections to the minimization of various f-divergences between \emph{two} distributions~\cite{nowozin2016f}.
However, for adversarial training among multiple distributions, despite similar idea has been empirically experimented~\cite{P17-1001}, no theoretical justifications have been provided to our best knowledge.

In this section, we present a theoretical analysis showing the validity of \amdc{}.
In particular, we show that \amdc{}'s objective is equivalent to minimizing the total f-divergence between each of the shared feature distributions of the $N$ domains, and the centroid of the $N$ distributions.
The choice of loss function will determine which specific f-divergence is minimized.
Furthermore, with adequate model capacity, \amdc{} achieves its optimum for either loss function if and only if all $N$ shared feature distributions are identical,
hence learning an invariant feature space across all domains.

First consider the distribution of the shared features $\fvec$ for instances in each domain $d_i\in\Delta$:
\begin{equation}
    P_i(\fvec) \triangleq P(\fvec=\cfs(x) | x\in d_i)
\end{equation}

Combining (\ref{eqn:jd}) with the two loss functions (\ref{eqn:ld-nll}), (\ref{eqn:ld-l2}), the objective of $\cd$ can be written as:
\begin{align}
\label{eqn:lossd-nll}
J^{NLL}_\cd &= - \sum_{i=1}^{N} \expe_{\fvec\sim P_i} \left[ \log \cd_i(\fvec) \right] \\
J^{L2}_\cd &= \sum_{i=1}^{N} \expe_{\fvec\sim P_i} \left[ \sum_{j=1}^N (\cd_j(\fvec)-\mathbbm{1}_{\{i=j\}})^2 \right] \label{eqn:lossd-l2}
\end{align}
where $\cd_i(\fvec)$ is the $i$-th dimension of $\cd$'s (normalized) output vector, which conceptually corresponds to the probability of $\cd$ predicting that $\fvec$ is from domain $d_i$.

We first derive the optimal $\cd$ for any fixed $\cfs$.
\begin{lemma}
For any fixed $\cfs$, with either NLL or L2 loss, the optimum domain discriminator $\cd^*$ is:
\begin{equation}
\label{eqn:dstar}
\cd_i^*(\fvec) = \frac{P_i(\fvec)}{\sum_{j=1}^N P_j(\fvec)}
\end{equation}
\end{lemma}
The proof involves an application of the \emph{Lagrangian Multiplier} to solve the minimum value of $J_\cd$, and the details can be found in the Appendix.
We then have the following main theorems for the domain loss for $\cfs$:
\begin{theorem}
Let $\overline{P} = \frac{\sum_{i=1}^N P_i}{N}$.
When $\cd$ is trained to its optimality, if $\cd$ adopts the \emph{NLL loss}:
\begin{align*}
J_{\cfs}^{\cd} &= -\min_{\theta_\cd} J_{\cd} = - J_{\cd^*} \\
&= -N\log N + N\cdot \mli{JSD}(P_1, P_2, \dots, P_N) \\
&= -N\log N + \sum_{i=1}^N \mli{KL}(P_i \| \overline{P})
\end{align*}
\end{theorem}
where $\mli{JSD}(\cdot)$ is the generalized Jensen-Shannon Divergence~\cite{61115} among multiple distributions,
defined as the average Kullback-Leibler divergence of each $P_i$ to the centroid $\overline{P}$~\cite{Aslam2007}.

\begin{theorem}
If $\cd$ uses the \emph{L2 loss}:
\begin{align*}
J_{\cfs}^{\cd} &= \sum_{i=1}^N \expe_{\fvec\sim P_i} \left[ \sum_{j=1}^N (\cd^*_j(\fvec)-\frac{1}{N})^2 \right] \\
&= \frac{1}{N} \sum_{i=1}^N \chi_{\mli{Neyman}}^2(P_i\| \overline{P})
\end{align*}
\end{theorem}
where $\chi_{\mli{Neyman}}^2(\cdot\|\cdot)$ is the Neyman $\chi^2$ divergence~\cite{nielsen2014chi}.
The proof of both theorems can be found in the Appendix.

Consequently, by the non-negativity and joint convexity of the f-divergence~\cite{Csiszar:1982:ITC:601016}, we have:
\begin{corollary}
The optimum of $J_{\cfs}^\cd$ is $-N\log N$ when using NLL loss, and $0$ for the L2 loss.
The optimum value above is achieved if and only if $P_1 = P_2 = \dots = P_N = \overline{P}$ for either loss.
\end{corollary}
Therefore, the loss of $\cfs$ can be interpreted as simultaneously minimizing the classification loss $J_\cc$ as well as the divergence among feature distributions of all domains.
It can thus learn a shared feature mapping that are invariant across domains upon successful training while being beneficial to the main classification task.

\section{Experiments}\label{sec:experiments}

\subsection{Multi-Domain Text Classification}\label{sec:mdtc_exp}
In this experiment, we compare \amdc{} to state-of-the-art MDTC systems, on the multi-domain Amazon review dataset~\cite{P07-1056}, which is one of the most widely used MDTC datasets.
Note that this dataset was already preprocessed into a bag of features (unigrams and bigrams), losing all word order information.
This prohibits the usage of CNNs or RNNs as feature extractors, limiting the potential performance of the system.
Nonetheless, we adopt the same dataset for fair comparison and employ a MLP as our feature extractor.
In particular, we take the 5000 most frequent features and represent each review as a 5000d feature vector, where feature values are raw counts of the features.
Our MLP feature extractor would then have an input size of 5000 in order to process the reviews.

The Amazon dataset contains 2000 samples for each of the four domains: \emph{book, DVD, electronics}, and \emph{kitchen}, with binary labels (positive, negative).
Following \newcite{wu2015}, we conduct 5-way cross validation.
Three out of the five folds are treated as training set, one serves as the validation set, while the remaining being the test set.
The 5-fold average test accuracy is reported.

\begin{table}
\small
\centering
\begin{threeparttable}[t]
\begin{tabular}{@{\hspace{0.2em}}c@{\hspace{0.2em}}|@{\hspace{0.3em}}c@{\hspace{0.3em}}|@{\hspace{0.3em}}c@{\hspace{0.3em}}|@{\hspace{0.3em}}c@{\hspace{0.3em}}|@{\hspace{0.3em}}c@{\hspace{0.3em}}|@{\hspace{0.3em}}c@{\hspace{0.3em}}}
\hline
&	Book	 &	DVD	& Elec.   &   Kit. &   Avg. \\
\hline
\multicolumn{6}{c}{Domain-Specific Models Only}\\
\hline
LS  & 77.80 & 77.88 & 81.63 & 84.33 & 80.41 \\
SVM & 78.56 & 78.66 & 83.03 & 84.74 & 81.25 \\
LR  & 79.73 & 80.14 & 84.54 & 86.10 & 82.63 \\
\textbf{MLP} & 81.70 & 81.65 & 85.45 & 85.95 & 83.69 \\
\hline
\multicolumn{6}{c}{Shared Model Only}\\
\hline
LS  & 78.40 & 79.76 & 84.67 & 85.73 & 82.14 \\
SVM & 79.16 & 80.97 & 85.15 & 86.06 & 82.83 \\
LR  & 80.05 & 81.88 & 85.19 & 86.56 & 83.42 \\
\textbf{MLP}  & 82.40 & 82.15 & 85.90 & 88.20 & 84.66 \\
\textbf{\amdc{}-L2-MLP}  & 82.05 & 83.45 & 86.45 & 88.85 & 85.20 \\
\textbf{\amdc{}-NLL-MLP}  & 81.85 & 83.10 & 85.75 & 89.10 & 84.95 \\
\hline
\multicolumn{6}{c}{Shared-Private Models}\\
\hline
RMTL\tnote{1}     & 81.33 & 82.18 & 85.49 & 87.02 & 84.01 \\
MTLGraph\tnote{2} & 79.66 & 81.84 & 83.69 & 87.06 & 83.06 \\
CMSC-LS\tnote{3}  & 82.10 & 82.40 & 86.12 & 87.56 & 84.55 \\
CMSC-SVM\tnote{3} & 82.26 & 83.48 & 86.76 & 88.20 & 85.18 \\
CMSC-LR\tnote{3}  & 81.81 & 83.73 & 86.67 & 88.23 & 85.11 \\
\textbf{SP-MLP}  & 82.00 & \textbf{84.05} & 86.85 & 87.30 & 85.05 \\
\textbf{\amdc{}-L2-SP-MLP} &
\begin{minipage}{0.8cm}{\centering 82.46\\ \vspace{-1mm}{\tiny ($\pm 0.25$)}}\vspace{1mm}\end{minipage} &
\begin{minipage}{0.8cm}{\centering 83.98\\ \vspace{-1mm}{\tiny ($\pm 0.17$)}}\vspace{1mm}\end{minipage} &
\begin{minipage}{0.8cm}{\centering \textbf{87.22*}\\ \vspace{-1mm}{\tiny ($\pm 0.04$)}}\vspace{1mm}\end{minipage} &
\begin{minipage}{0.8cm}{\centering 88.53\\ \vspace{-1mm}{\tiny ($\pm 0.19$)}}\vspace{1mm}\end{minipage} &
\begin{minipage}{0.8cm}{\centering 85.55*\\ \vspace{-1mm}{\tiny ($\pm 0.07$)}}\vspace{1mm}\end{minipage} \\
\textbf{\amdc{}-NLL-SP-MLP} & 
\begin{minipage}{0.8cm}{\centering \textbf{82.98*}\\ \vspace{-1mm}{\tiny ($\pm 0.28$)}}\end{minipage} &
\begin{minipage}{0.8cm}{\centering 84.03\\ \vspace{-1mm}{\tiny ($\pm 0.16$)}}\end{minipage} &
\begin{minipage}{0.8cm}{\centering 87.06\\ \vspace{-1mm}{\tiny ($\pm 0.23$)}}\end{minipage} &
\begin{minipage}{0.8cm}{\centering \textbf{88.57*}\\ \vspace{-1mm}{\tiny ($\pm 0.15$)}}\end{minipage} &
\begin{minipage}{0.8cm}{\centering \textbf{85.66*}\\ \vspace{-1mm}{\tiny ($\pm 0.14$)}}\end{minipage} \\
\hline
\end{tabular}
\begin{tablenotes}
\item[1] \newcite{Evgeniou:2004:RML:1014052.1014067}
\item[2] \newcite{zhou2012mutal}
\item[3] \newcite{wu2015}
\end{tablenotes}
\caption{
MDTC results on the Amazon dataset.
Models in bold are ours while the performance of the rest is taken from \newcite{wu2015}.
Numbers in parentheses indicate standard errors, calculated based on 5 runs.
Bold numbers indicate the highest performance in each domain, and $*$ shows statistical significance ($p<0.05$) over CMSC under a one-sample T-Test.
}
\label{tab:amazon5fold}
\end{threeparttable}
\end{table}

Table~\ref{tab:amazon5fold} shows the main results.
Three types of models are shown:
\emph{Domain-Specific Models Only}, where only in-domain models are trained\footnote{For our models, it means $\cfs$ is disabled. Similarly, for Shared Model Only, no $\cfd$ is used.};
\emph{Shared Model Only}, where a single model is trained with all data;
and \emph{Shared-Private Models}, a combination of the previous two.
Within each category, various architectures are examined, such as Least Square (LS), SVM, and Logistic Regression (LR).
As explained before, we use MLP as our feature extractors for all our models (bold ones).
Among our models, the ones with the \amdc{} prefix use adversarial training, and \amdc{}-L2 and \amdc{}-NLL indicate the L2 loss and NLL loss \amdc{}, respectively.

From Table~\ref{tab:amazon5fold}, we can see that by adopting modern deep neural networks, our methods achieve superior performance within the first two model categories even without adversarial training.
This is corroborated by the fact that our SP-MLP model performs comparably to CMSC, while the latter relies on external resources such as sentiment lexica.
Moreover, when our multinomial adversarial nets are introduced, further improvement is observed.
With both loss functions, \amdc{} outperforms all Shared-Private baseline systems on each domain, and achieves statistically significantly higher overall performance.
For our \amdc{}-SP models, we provide the mean accuracy as well as the standard errors over five runs, to illustrate the performance variance and conduct significance test.
It can be seen that \amdc{}'s performance is relatively stable, and consistently outperforms CMSC.

\subsection{Experiments for Unlabeled Domains}\label{sec:multisource_exp}

As CMSC requires labeled data for each domain, their experiments were naturally designed this way.
In reality, however, many domains may not have any annotated corpora available.
It is therefore also important to look at the performance in these unlabeled domains for a MDTC system.
Fortunately, as depicted before, \amdc{}'s adversarial training only utilizes unlabeled data from each domain to learn the domain-invariant features, and can thus be used on unlabeled domains as well.
During testing, only the shared feature vector is fed into $\cc$, while the domain feature vector is set to $\bf 0$.

\begin{table}[t]
\small
\centering
\begin{threeparttable}
\begin{tabular}{@{\hspace{0.2em}}c@{\hspace{0.2em}}|@{\hspace{0.3em}}c@{\hspace{0.3em}}|@{\hspace{0.3em}}c@{\hspace{0.3em}}|@{\hspace{0.3em}}c@{\hspace{0.3em}}|@{\hspace{0.3em}}c@{\hspace{0.3em}}|@{\hspace{0.3em}}c@{\hspace{0.3em}}}
\hline
Target Domain	&	Book	 &	DVD	& Elec.   &   Kit. &   Avg. \\
\hline
MLP  & 76.55 & 75.88 & 84.60 & 85.45 & 80.46 \\
mSDA\tnote{1} & 76.98 & 78.61 & 81.98 & 84.26 & 80.46 \\
DANN\tnote{2}  & 77.89 & 78.86 & 84.91 & 86.39 & 82.01 \\
MDAN (H-MAX)\tnote{3}  & 78.45 & 77.97 & 84.83 & 85.80 & 81.76 \\
MDAN (S-MAX)\tnote{3}  & \textbf{78.63} & 80.65 & \textbf{85.34} & 86.26 & \textbf{82.72} \\
\textbf{\amdc{}-L2-SP-MLP}  & 78.45 & 81.57 & 83.37 & 85.57 & 82.24 \\
\textbf{\amdc{}-NLL-SP-MLP}  & 77.78 & \textbf{82.74} & 83.75 & \textbf{86.41} &  82.67\\
\hline
\end{tabular}
\begin{tablenotes}
\item[1] \newcite{ICML2012Chen_416}
\item[2] \newcite{ganin2016domain}
\item[3] \newcite{DBLP:journals/corr/ZhaoZWCMG17}
\end{tablenotes}
\end{threeparttable}
\caption{
Results on unlabeled domains.
Models in bold are our models while the rest is taken from \newcite{DBLP:journals/corr/ZhaoZWCMG17}.
Highest domain performance is shown in bold.
}
\label{tab:amazon-adapt}
\end{table}

In order to validate \amdc{}'s effectiveness, we compare to state-of-the-art \emph{multi-source domain adaptation} (MS-DA) methods (See \secref{sec:relatedwork}).
Compared to standard domain adaptation methods with one source and one target domain, MS-DA allows the adaptation from multiple source domains to a single target domain.
Analogically, MDTC can be viewed as \emph{multi-source multi-target} domain adaptation, which is superior when multiple target domains exist.
With multiple target domains, MS-DA will need to treat each one as an independent task, which is more expensive and cannot utilize the unlabeled data in other target domains.

In this work, we compare \amdc{} with one recent MS-DA method, MDAN~\cite{DBLP:journals/corr/ZhaoZWCMG17}.
Their experiments only have one target domain to suit their approach, and we follow this setting for fair comparison.
However, it is worth noting that \amdc{} is designed for the MDTC setting, and can deal with multiple target domains at the same time, which can potentially improve the performance by taking advantage of more unlabeled data from multiple target domains during adversarial training.
We adopt the same setting as~\newcite{DBLP:journals/corr/ZhaoZWCMG17}, which is based on the same multi-domain Amazon review dataset.
Each of the four domains in the dataset is treated as the target domain in four separate experiments, while the remaining three are used as source domains.

\begin{table*}[t]
\footnotesize
\centering
\begin{tabular}{@{\hspace{0.15em}}c@{\hspace{0.15em}}|@{\hspace{0.15em}}c@{\hspace{0.15em}}|@{\hspace{0.15em}}c@{\hspace{0.15em}}|@{\hspace{0.1em}}c@{\hspace{0.1em}}|@{\hspace{0.1em}}c@{\hspace{0.1em}}|@{\hspace{0.1em}}c@{\hspace{0.1em}}|@{\hspace{0.1em}}c@{\hspace{0.1em}}|@{\hspace{0.1em}}c@{\hspace{0.1em}}|@{\hspace{0.15em}}c@{\hspace{0.15em}}|@{\hspace{0.15em}}c@{\hspace{0.15em}}|@{\hspace{0.15em}}c@{\hspace{0.15em}}|@{\hspace{0.15em}}c@{\hspace{0.15em}}|@{\hspace{0.15em}}c@{\hspace{0.15em}}|@{\hspace{0.15em}}c@{\hspace{0.15em}}|@{\hspace{0.15em}}c@{\hspace{0.15em}}|@{\hspace{0.15em}}c@{\hspace{0.15em}}|@{\hspace{0.15em}}c@{\hspace{0.15em}}|@{\hspace{0.15em}}c@{\hspace{0.15em}}}
\hline
& \scriptsize books & \scriptsize elec. & \scriptsize dvd & \scriptsize kitchen &\scriptsize  apparel &\scriptsize  camera &\scriptsize  health &\scriptsize  music &\scriptsize  toys &\scriptsize  video &\scriptsize  baby &\scriptsize  magaz. &\scriptsize  softw. &\scriptsize  sports &\scriptsize  IMDb &\scriptsize  MR & \scriptsize Avg.\\
\hline
\multicolumn{18}{c}{Domain-Specific Models Only}\\
\hline
BiLSTM       & 81.0 & 78.5 & 80.5 & 81.2 & 86.0 & 86.0 & 78.7 & 77.2 & 84.7 & 83.7 & 83.5 & 91.5 & 85.7 & 84.0 & 85.0 & 74.7 & 82.6 \\
\textbf{CNN} & 85.3 & 87.8 & 76.3 & 84.5 & 86.3 & 89.0 & 87.5 & 81.5 & 87.0 & 82.3 & 82.5 & 86.8 & 87.5 & 85.3 & 83.3 & 75.5 & 84.3 \\
\hline
\multicolumn{18}{c}{Shared Model Only}\\
\hline
FS-MTL               & 82.5 & 85.7 & 83.5 & 86.0 & 84.5 & 86.5 & 88.0 & 81.2 & 84.5 & 83.7 & 88.0 & 92.5 & 86.2 & 85.5 & 82.5 & 74.7 & 84.7 \\
\textbf{\amdc{}-L2-CNN} & \textbf{88.3} & 88.3 & 87.8 & 88.5 & 85.3 & 90.5 & \textbf{90.8} & 85.3 & 89.5 & 89.0 & 89.5 & 91.3 & 88.3 & 89.5 & \textbf{88.5} & 73.8 & 87.7 \\
\textbf{\amdc{}-NLL-CNN} & 88.0 & 87.8 & 87.3 & 88.5 & 86.3 & 90.8 & 89.8 & 84.8 & 89.3 & 89.3 & 87.8 & 91.8 & 90.0 & \textbf{90.3} & 87.3 & 73.5 & 87.6 \\
\hline
\multicolumn{18}{c}{Shared-Private Models}\\
\hline
ASP-MTL                     & 84.0 & 86.8 & 85.5 & 86.2 & 87.0 & 89.2 & 88.2 & 82.5 & 88.0 & 84.5 & 88.2 & 92.2 & 87.2 & 85.7 & 85.5 & \textbf{76.7} & 86.1 \\
\multirow{2}{*}{\textbf{\amdc{}-L2-SP-CNN}}  & 87.6* & 87.4 & 88.1* & 89.8* & \textbf{87.6} & \textbf{91.4}* & 89.8* & \textbf{85.9}* & 90.0* & 89.5* & 90.0 & 92.5 & 90.4* & 89.0* & 86.6 & 76.1 & 88.2* \\[-1ex]
& \tiny $( 0.2)$ & \tiny $( 1.0)$ & \tiny $( 0.4)$ & \tiny $( 0.4)$ & \tiny $( 0.7)$ & \tiny $( 0.4)$ & \tiny $( 0.3)$ & \tiny $( 0.1)$ & \tiny $( 0.1)$ & \tiny $( 0.2)$ & \tiny $( 0.6)$ & \tiny $( 0.5)$ & \tiny $( 0.4)$ & \tiny $( 0.4)$ & \tiny $( 0.5)$ & \tiny $( 0.5)$ & \tiny $( 0.1)$ \\
\multirow{2}{*}{\textbf{\amdc{}-NLL-SP-CNN}} & 86.8* & \textbf{88.8} & \textbf{88.6}* & \textbf{89.9}* & \textbf{87.6} & 90.7 & 89.4 & 85.5* & \textbf{90.4}* & \textbf{89.6}* & \textbf{90.2} & \textbf{92.9} & \textbf{90.9}* & 89.0* & 87.0* & \textbf{76.7} & \textbf{88.4}* \\[-1ex]
& \tiny $( 0.4)$ & \tiny $( 0.6)$ & \tiny $( 0.4)$ & \tiny $( 0.4)$ & \tiny $( 0.4)$ & \tiny $( 0.4)$ & \tiny $( 0.3)$ & \tiny $( 0.1)$ & \tiny $( 0.2)$ & \tiny $( 0.3)$ & \tiny $( 0.6)$ & \tiny $( 0.4)$ & \tiny $( 0.7)$ & \tiny $( 0.2)$ & \tiny $( 0.1)$ & \tiny $( 0.8)$ & \tiny $( 0.1)$ \\
\hline
\end{tabular}
\caption{
Results on the FDU-MTL dataset.
Bolded models are ours, while the rest is from~\newcite{P17-1001}.
Highest performance is each domain is highlighted.
For our full \amdc{} models, standard errors are shown in parenthese and statistical significance ($p<0.01$) over ASP-MTL is indicated by *.
}
\label{tab:fdu_mtl}
\end{table*}

In Table~\ref{tab:amazon-adapt}, the target domain is shown on top, and the test set accuracy is reported for various systems.
It shows that \amdc{} outperforms several baseline systems, such as a MLP trained on the source-domains, as well as single-source domain adaptation methods such as mSDA~\cite{ICML2012Chen_416} and DANN~\cite{ganin2016domain}, where the training data in the multiple source domains are combined and viewed as a single domain.
Finally, when compared to MDAN, \amdc{} and MDAN each achieves higher accuracy on two out of the four target domains, and the average accuracy of \amdc{} is similar to MDAN.
Therefore, \amdc{} achieves competitive performance for the domains without annotated corpus.
Nevertheless, unlike MS-DA methods, \amdc{} can handle multiple target domains at one time.

\subsection{Experiments on the MTL Dataset}\label{sec:mtl_exp}

To make fair comparisons, the previous experiments follow the standard settings in the literature, where the widely adopted Amazon review dataset is used.
However, this dataset has a few limitations:
First, it has only four domains.
In addition, the reviews are already tokenized and converted to a bag of features consisting of unigrams and bigrams.
Raw review texts are hence not available in this dataset, making it impossible to use certain modern neural architectures such as CNNs and RNNs.
To provide more insights on how well \amdc{} work with other feature extractor architectures, we provide a third set of experiments on the FDU-MTL dataset~\cite{P17-1001}.
The dataset is created as a multi-task learning dataset with 16 \emph{tasks}, where each task is essentially a different domain of reviews.
It has 14 Amazon domains: books, electronics, DVD, kitchen, apparel, camera, health, music, toys, video, baby, magazine, software, and sports, in addition to two movies review domains from the IMDb and the MR dataset.
Each domain has a development set of 200 samples, and a test set of 400 samples.
The amount of training and unlabeled data vary across domains but are roughly 1400 and 2000, respectively.

We compare \amdc{} with ASP-MTL~\cite{P17-1001} on this FDU-MTL dataset.
ASP-MTL also adopts adversarial training for learning a shared feature space, and can be viewed as a special case of \amdc{} when adopting the NLL loss (\amdc{}-NLL).
Furthermore, while \newcite{P17-1001} do not provide any theoretically justifications, we in \secref{sec:model_proof} prove the validity of \amdc{} for not only the NLL loss, but an additional L2 loss.
Besides the theoretical superiority, we in this section show that \amdc{} also substantially outperforms ASP-MTL in practice due to the feature extractor choice. 

In particular, \newcite{P17-1001} choose LSTM as their feature extractor, yet we found CNN~\cite{D14-1181} to achieve much better accuracy while being $\sim 10$ times faster.
Indeed, as shown in Table~\ref{tab:fdu_mtl}, with or without adversarial training, our CNN models outperform LSTM ones by a large margin.
When \amdc{} is introduced, we attain the state-of-the-art performance on every domain with a 88.4\% overall accuracy, surpassing ASP-MTL by a significant margin of 2.3\%.

We hypothesize the reason LSTM performs much inferior to CNN is attributed to the lack of attention mechanism.
In ASP-MTL, only the last hidden unit is taken as the extracted features.
While LSTM is effective for representing the context for each token, it might not be powerful enough for directly encoding the entire document~\cite{DBLP:journals/corr/BahdanauCB14}.
Therefore, various attention mechanisms have been introduced on top of the vanilla LSTM to select words (and contexts) most relevant for making the predictions.
In our preliminary experiments, we find that Bi-directional LSTM with the dot-product attention~\cite{D15-1166} yields better performance than the vanilla LSTM in ASP-MTL.
However, it still does not outperform CNN and is much slower.
As a result, we conclude that, for text classification tasks, CNN is both effective and efficient in extracting local and higher-level features for making a single categorization.

Finally, we observe that \amdc{}-NLL achieves slightly higher overall performance compared to \amdc{}-L2, providing evidence for the claim in a recent study~\cite{2017arXiv171110337L} that the original GAN loss (NLL) may not be inherently inferior.
Moreover, the two variants excel in different domains, suggesting the possibility of further performance gain when using ensemble.
\section{Related Work}\label{sec:relatedwork}
\paragraph{Multi-Domain Text Classification}
The MDTC task was first examined by~\newcite{P08-2065}, who proposed to fusion the training data from multiple domains either on the feature level or the classifier level.
The prior art of MDTC~\cite{wu2015} decomposes the text classifier into a general one and a set of domain-specific ones. However, the general classifier is learned by parameter sharing and domain-specific knowledge may sneak into it.
They also require external resources to help improve accuracy and compute domain similarities.

\paragraph{Domain Adaptation}
Domain Adaptation attempts to transfer the knowledge from a source domain to a target one, and the traditional form is the \emph{single-source, single-target} (\textbf{SS,ST}) adaptation~\cite{W06-1615}.
Another variant is the \textbf{SS,MT} adaptation~\cite{N15-1069}, which tries to simultaneously transfer the knowledge to multiple target domains from a single source.
However, it cannot fully take advantage the training data if it comes from multiple source domains.
\textbf{MS,ST} adaptation~\cite{NIPS2008_3550,DBLP:journals/corr/ZhaoZWCMG17} can deal with multiple source domains but only transfers to a single target domain.
Therefore, when multiple target domains exist, they need to treat them as independent problems, which is more expensive and cannot utilize the additional unlabeled data in these domains.
Finally, MDTC can be viewed as \textbf{MS,MT} adaptation, which is arguably more general and realistic.

\paragraph{Adversarial Networks}
The idea of adversarial networks was proposed by~\newcite{NIPS2014_5423} for image generation, and has been applied to various NLP tasks as well~\cite{2016arXiv160601614C,D17-1230}.
\newcite{ganin2016domain} first used it for the \textbf{SS,ST} domain adaptation followed by many others.
\newcite{NIPS2016_6254} utilized adversarial training in a shared-private model for domain adaptation to learn domain-invariant features, but still focused on the \textbf{SS,ST} setting.
Finally, the idea of using adversarial nets to discriminate over multiple distributions was empirically explored by a very recent work~\cite{P17-1001} under the multi-task learning setting, and can be considered as a special case of our \amdc{} framework with the NLL domain loss.
Nevertheless, we propose a more general framework with alternative architectures for the adversarial component, and for the first time provide theoretical justifications for the multinomial adversarial nets.
Moreover, \newcite{P17-1001} used LSTM without attention as their feature extractor, which we found to perform sub-optimal in the experiments. We instead chose Convolutional Neural Nets as our feature extractor that achieves higher accuracy while running an order of magnitude faster (See \secref{sec:mtl_exp}).
\section{Conclusion}
In this work, we propose a family of Multinomial Adversarial Networks (\amdc{}) that generalize the traditional binomial adversarial nets in the sense that \amdc{} can simultaneously minimize the difference among multiple probability distributions instead of two.
We provide theoretical justifications for two instances of \amdc{}, \amdc{}-NLL and \amdc{}-L2, showing they are minimizers of two different f-divergence metrics among multiple distributions, respectively.
This indicates \amdc{} can be used to make multiple distributions indistinguishable from one another.
It can hence be applied to a variety of tasks, similar to the versatile binomial adversarial nets, which have been used in many areas for making \emph{two} distributions alike.

We in this paper design a \amdc{} model for the MDTC task, following the shared-private paradigm that has a shared feature extractor to learn domain-invariant features and domain feature extractors to learn domain-specific ones.
\amdc{} is used to enforce the shared feature extractor to learn only domain-invariant knowledge, by resorting to \amdc{}'s power of making indistinguishable the shared feature distributions of samples from each domain.
We conduct extensive experiments, demonstrating our \amdc{} model outperforms the prior art systems in MDTC, and achieves state-of-the-art performance on domains without labeled data when compared to multi-source domain adaptation methods.

\bibliography{amdc}
\bibliographystyle{acl_natbib}

\newpage
\onecolumn
\begin{appendices}
\section{Proofs}
\subsection{Proofs for \amdc{}-NLL}

Assume we have $N$ domains, consider the distribution of the shared features $\cfs$ for instances in each domain $d_i$:
\begin{equation*}
    P_i(\fvec) \triangleq P(\fvec=\cfs(x) | x\in d_i)
\end{equation*}

The objective that $\cd$ attempts to minimize is:
\begin{equation}
\label{apd-eqn:lossd-nll}
J_\cd = - \sum_{i=1}^{N} \expe_{\fvec\sim P_i} \left[ \log \cd_i(\fvec) \right]
\end{equation}
where $\cd_i(\fvec)$ is the $i$-th dimension of $\cd$'s output vector, which conceptually corresponds to the softmax probability of $\cd$ predicting that $\fvec$ is from domain $d_i$.
We therefore have property that for any $\fvec$:
\begin{equation}
\label{apd-eqn:sumd-nll}
\sum_{i=1}^N \cd_i(\fvec) = 1
\end{equation}

\begin{lemma}
For any fixed $\cfs$, the optimum domain discriminator $\cd^*$ is:
\begin{equation}
\label{apd-eqn:dstar-nll}
\cd_i^*(\fvec) = \frac{P_i(\fvec)}{\sum_{j=1}^N P_j(\fvec)}
\end{equation}
\end{lemma}

\begin{proof}
For a fixed $\cfs$, the optimum
\begin{align*}
    \cd^* = \argmin_\cd J_\cd &= \argmin_\cd -\sum_{i=1}^{N} \expe_{\fvec\sim P_i} \left[ \log \cd_i(\fvec) \right]\\
    &= \argmax_\cd \sum_{i=1}^N \int_\fvec P_i(\fvec) \log \cd_i(\fvec) d\fvec \\
    &= \argmax_\cd \int_\fvec \sum_{i=1}^N P_i(\fvec) \log \cd_i(\fvec) d\fvec
\end{align*}
We employ the Lagrangian Multiplier to derive $\argmax_\cd \sum_{i=1}^N P_i(\fvec) \log \cd_i(\fvec)$ under the constraint of (\ref{apd-eqn:sumd-nll}).
Let
\begin{align*}
    L(\cd_1, \dots, \cd_N, \lambda) &= \sum_{i=1}^N P_i\log \cd_i - \lambda(\sum_{i=1}^N \cd_i - 1)
\end{align*}
Let $\nabla L = 0$:
\begin{align*}
\begin{cases}
\nabla_{\cd_i}\sum_{j=1}^N P_j\log \cd_j = \lambda \nabla_{\cd_i}(\sum_{j=1}^N \cd_j - 1) & (\forall i)\\
\sum_{i=1}^N \cd_i - 1 = 0
\end{cases}
\end{align*}
Solving the two equations, we have:
$$
\cd_i^*(\fvec) = \frac{P_i(\fvec)}{\sum_{j=1}^N P_j(\fvec)}
$$

\end{proof}

On the other hand, the loss function of the shared feature extractor $\cfs$ consists of two additive components, the loss from the text classifier $\cc$, and the loss from the domain discriminator $\cd$:
\begin{equation}
\label{apd-eqn:lossfs}
J_{\cfs} = J_{\cfs}^{\cc} + \lambda J_{\cfs}^{\cd} \triangleq J_\cc - \lambda J_\cd
\end{equation}

We have the following theorem for the domain loss for $\cfs$:
\begin{theorem}
When $\cd$ is trained to its optimality:
\begin{equation}
\label{apd-eqn:fdloss-jsd}
J_{\cfs}^{D} = -J_{\cd^*} = -N\log N + N\cdot \mli{JSD}(P_1, P_2, \dots, P_N)
\end{equation}
\end{theorem}
where $\mli{JSD}(\cdot)$ is the generalized Jensen-Shannon Divergence~\cite{61115} among multiple distributions.
\begin{proof}
Let $\overline{P} = \frac{\sum_{i=1}^N P_i}{N}$.

There are two equivalent definitions of the generalized Jensen-Shannon divergence: the original definition based on Shannon entropy~\cite{61115}, and a reshaped one expressed as the average Kullback-Leibler divergence of each $P_i$ to the centroid $\overline{P}$~\cite{Aslam2007}.
We adopt the latter one here:
\begin{equation}
\mli{JSD}(P_1, P_2, \dots, P_N) \triangleq \frac{1}{N} \sum_{i=1}^N \mli{KL}(P_i \| \overline{P})
= \frac{1}{N} \sum_{i=1}^N \expe_{\fvec\sim P_i}\left[ \log \frac{P_i(\fvec)}{\overline{P}(\fvec)} \right]
\end{equation}

Now substituting $\cd^*$ into $J_{\cfs}^\cd$:
\begin{align*}
J_{\cfs}^{\cd} = -J_{\cd^*}
&= \sum_{i=1}^{N} \expe_{\fvec\sim P_i} \left[ \log \cd^*_i(\fvec) \right] \\
&= \sum_{i=1}^{N} \expe_{\fvec\sim P_i} \left[ \log \frac{P_i(\fvec)}{\sum_{j=1}^N P_j(\fvec)} \right] \\
&= -N\log N + \sum_{i=1}^{N} \expe_{\fvec\sim P_i} \left[ \log \frac{P_i(\fvec)}{\sum_{j=1}^N P_j(\fvec)} + \log N \right] \\
&= -N\log N + \sum_{i=1}^{N} \expe_{\fvec\sim P_i} \left[ \log \frac{P_i(\fvec)}{\frac{\sum_{j=1}^N P_j(\fvec)}{N}} \right] \\
&= -N\log N + \sum_{i=1}^{N} \expe_{\fvec\sim P_i} \left[ \log \frac{P_i(\fvec)}{\overline{P}} \right] \\
&= -N\log N + \sum_{i=1}^{N} \mli{KL}(P_i\|\overline{P}) \\
&= -N\log N + N\cdot \mli{JSD}(P_1, P_2, \dots, P_N)
\end{align*}
\end{proof}

Consequently, by the non-negativity of $\mli{JSD}$~\cite{61115}, we have the following corollary:
\begin{corollary}
The optimum of $J_{\cfs}^\cd$ is $-N\log N$, and is achieved if and only if $P_1 = P_2 = \dots = P_N = \overline{P}$.
\end{corollary}

\subsection{Proofs for \amdc{}-L2}
The proof is similar for \amdc{} with the L2 loss.
The loss function used by $\cd$ is, for a sample from domain $d_i$ with shared feature vector $\fvec$:
\begin{equation}
\label{apd-eqn:lossd-l2}
\cl_\cd(\cd(\fvec), i) = \sum_{j=1}^N (\cd_j(\fvec)-\mathbbm{1}_{\{i=j\}})^2
\end{equation}
So the objective that $\cd$ minimizes is:
\begin{equation}
\label{apd-eqn:jd}
J_\cd = \sum_{i=1}^{N} \expe_{\fvec\sim P_i} \left[ \sum_{j=1}^N (\cd_j(\fvec)-\mathbbm{1}_{\{i=j\}})^2 \right]
\end{equation}
For simplicity, we further constrain $\cd$'s outputs to be on a simplex:
\begin{equation}
\label{apd-eqn:sumd-l2}
\sum_{i=1}^N \cd_i(\fvec) = 1 \quad (\forall \fvec)
\end{equation}

\begin{lemma}
For any fixed $\cfs$, the optimum domain discriminator $\cd^*$ is:
\begin{equation}
\label{apd-eqn:dstar-l2}
\cd_i^*(\fvec) = \frac{P_i(\fvec)}{\sum_{j=1}^N P_j(\fvec)}
\end{equation}
\end{lemma}

\begin{proof}
For a fixed $\cfs$, the optimum
\begin{align*}
    \cd^* = \argmin_\cd J_\cd &= \argmin_\cd \sum_{i=1}^{N} \expe_{\fvec\sim P_i} \left[ \cl_\cd(\cd(\fvec), i) \right]\\
    &= \argmin_\cd \sum_{i=1}^N \int_\fvec P_i(\fvec) \cl_\cd(\cd(\fvec), i) d\fvec \\
    &= \argmin_\cd \int_\fvec \sum_{i=1}^N P_i(\fvec) \sum_{j=1}^N (\cd_j(\fvec)-\mathbbm{1}_{\{i=j\}})^2 d\fvec
\end{align*}
Similar to \amdc{}-NLL, we employ the Lagrangian Multiplier to derive $\argmax_\cd \sum_{i=1}^N P_i(\fvec)\sum_{j=1}^N (\cd_j(\fvec)-\mathbbm{1}_{\{i=j\}})^2 $ under the constraint of (\ref{apd-eqn:sumd-l2}).
Let $\nabla L = 0$:
\begin{align*}
\begin{cases}
2((\sum_{j=1}^N P_j)\cd_i - P_i) = \lambda & (\forall i)\\
\sum_{i=1}^N \cd_i - 1 = 0
\end{cases}
\end{align*}
Solving the two equations, we have $\lambda=0$ and:
$$
\cd_i^*(\fvec) = \frac{P_i(\fvec)}{\sum_{j=1}^N P_j(\fvec)}
$$

\end{proof}


For the domain loss of $\cfs$:
\begin{theorem}
Let $\overline{P} = \frac{\sum_{i=1}^N P_i}{N}$.
When $\cd$ is trained to its optimality:
\begin{align}
\label{apd-eqn:fdloss-chi2}
J_{\cfs}^{\cd} &=  \sum_{i=1}^N \expe_{\fvec\sim P_i} \left[ \sum_{j=1}^N (D_j(\fvec)-\frac{1}{N})^2 \right] \nonumber\\
&= \frac{1}{N} \sum_{i=1}^N \chi_{\mli{Neyman}}^2(P_i\| \overline{P})
\end{align}
\end{theorem}
where $\chi_{\mli{Neyman}}^2(\cdot \| \cdot)$ is the Neyman $\chi^2$ divergence~\cite{nielsen2014chi}.
\begin{proof}
Substituting $\cd^*$ into $\cl_{\cfs}^\cd$:
\begin{align*}
J_{\cfs}^{\cd} &= \sum_{i=1}^N \expe_{\fvec\sim P_i} \left[ \sum_{j=1}^N (D^*_j(\fvec)-\frac{1}{N})^2 \right]\\
&= \sum_{i=1}^N \int_\fvec P_i\sum_{j=1}^N (\frac{P_j}{N\overline{P}}-\frac{1}{N})^2 d\fvec\\
&= \int_\fvec \sum_{i=1}^N \sum_{j=1}^N P_i(\frac{P_j}{N\overline{P}}-\frac{1}{N})^2 d\fvec\\
&= \frac{1}{N^2}\sum_{j=1}^N \int_\fvec \sum_{i=1}^N P_i(\frac{P_j}{\overline{P}}-1)^2 d\fvec\\
&= \frac{1}{N^2}\sum_{j=1}^N \int_\fvec N\overline{P}(\frac{P_j}{\overline{P}}-1)^2 d\fvec\\
&= \frac{1}{N}\sum_{j=1}^N \int_\fvec \frac{(P_j - \overline{P})^2}{\overline{P}} d\fvec\\
&= \frac{1}{N} \sum_{i=1}^N \chi_{\mli{Neyman}}^2(P_i\| \overline{P})
\end{align*}
\end{proof}

Finally, by the joint convexity of f-divergence, we have the following corollary:
\begin{corollary}
\begin{align*}
\cl_{\cfs}^\cd &= \frac{1}{N} \sum_{i=1}^N \chi_{\mli{Neyman}}^2(P_i\| \overline{P}) \\
&\geq \chi_{\mli{Neyman}}^2(\frac{1}{N}\sum_{i=1}^N P_i\| \frac{1}{N}\sum_{i=1}^N \overline{P}) \\
&= \chi_{\mli{Neyman}}^2( \overline{P}\| \overline{P}) = 0
\end{align*}
and the equality is attained if and only if $P_1 = P_2 = \dots = P_N = \overline{P}$.
\end{corollary}

\section{Implementation Details}
For all three of our experiments, we use $\lambda=0.05$ and $k=5$ (See Algorithm~\ref{alg:training}).
For both optimizers, Adam~\cite{kingma2014adam} is used with learning rate $0.0001$.
The size of the shared feature vector is set to $128$ while that of the domain feature vector is $64$.
Dropout of $p=0.4$ is used in all components.
$\cc$ and $\cd$ each has one hidden layer of the same size as their input ($128+64$ for $\cc$ and $128$ for $\cd$).
ReLU is used as the activation function.
Batch normalization~\cite{43442} is used in both $\cc$ and $\cd$ but not $\mathcal{F}$.
We use a batch size of 8.

For our first two experiments on the Amazon review dataset, the MLP feature extractor is used.
As described in the paper, it has an input size of 5000.
Two hidden layers are used, with size $1000$ and $500$, respectively.

For the CNN feature extractor used in the FDU-MTL experiment, a single convolution layer is used.
The kernel sizes are 3, 4, and 5, and the number of kernels are 200.
The convolution layers take as input the 100d word embeddings of each word in the input sequence.
We use \emph{word2vec} word embeddings~\cite{mikolov2013efficient} trained on a bunch of unlabeled raw Amazon reviews~\cite{P07-1056}.
After convolution, the outputs go through a ReLU layer before fed into a max pooling layer.
The pooled output is then fed into a single fully connected layer to be converted into a feature vector of size either 128 or 64.
More details of using CNN for text classification can be found in the original paper~\cite{D14-1181}.
\amdc{} is implemented using PyTorch~\cite{paszke2017automatic}.
\end{appendices}

\end{document}